\documentclass{llncs}

\usepackage{color}
\usepackage{algpseudocode}
\algtext*{EndIf}
\algtext*{EndFor}
\algtext*{EndWhile}
\algtext*{EndFunction}
\usepackage{float}
\usepackage{multirow}
\algnewcommand{\LeftComment}[1]{\Statex \(\triangleright\) #1}
	\usepackage{tabularx}
	\usepackage{pbox}

\usepackage{hyperref}
\usepackage{url}
\usepackage{tikz}
\usetikzlibrary{arrows}
\usetikzlibrary{shapes,fit}

\usepackage{amsmath}
\newcommand{\suppress}[1]{}
\newcommand{\todo}[1]{{\color{red} TODO: {#1}}}

\newenvironment{sketch}{%
	\proof}{\endproof}

\newtheorem{innercustomgeneric}{\customgenericname}
\providecommand{\customgenericname}{}
\newcommand{\newcustomtheorem}[2]{%
	\newenvironment{#1}[1]
	{%
		\renewcommand\customgenericname{#2}%
		\renewcommand\theinnercustomgeneric{##1}%
		\innercustomgeneric
	}
	{\endinnercustomgeneric}
}

\newcustomtheorem{customthm}{Theorem}
\newcustomtheorem{customlemma}{Lemma}

\newcommand{\set}[1]{\left\{ #1 \right\}}
\newcommand{\st}{\medspace | \medspace}

\newcommand{\spcnoindent}{\vspace{1em} \noindent}

\title{Relating Complexity-theoretic Parameters with SAT
  Solver Performance}

\author{
	Edward Zulkoski\inst{1}
	\and
	Ruben Martins\inst{2}
	\and
	Christoph M. Wintersteiger\inst{3}
	\and
	Robert Robere\inst{4}
	\and
	Jia Liang\inst{1}
	\and
	Krzysztof Czarnecki\inst{1}
	\and
	Vijay Ganesh\inst{1}
}

\institute{
	University of Waterloo,
	Waterloo, Ontario, Canada\\
	\email{ezulkosk@uwaterloo.ca}
	\and
	University of Texas at Austin,
	Austin, Texas, U.S.A.\\
	\and
	Microsoft Research,
	Cambridge, U.K.\\
	\and
	University of Toronto,
	Toronto, Ontario, Canada \\
}

\begin{document}

\maketitle

\begin{abstract}
  Over the years complexity theorists have proposed many structural
  parameters to explain the surprising efficiency of conflict-driven
  clause-learning (CDCL) SAT solvers on a wide variety of large
  industrial Boolean instances. While some of these parameters have
  been studied empirically, until now there has not been a unified
  comparative study of their explanatory power on a comprehensive
  benchmark. We correct this state of affairs by conducting a
  large-scale empirical evaluation of CDCL SAT solver performance on
  nearly 7000 industrial and crafted formulas against several
  structural parameters such as backdoors, treewidth, backbones, and
  community structure.

  Our study led us to several results. First, we show that while such
  parameters only weakly correlate with CDCL solving time, certain
  combinations of them yield much better regression models. Second, we
  show how some parameters can be used as a ``lens'' to better
  understand the efficiency of different solving heuristics. Finally,
  we propose a new complexity-theoretic parameter, which we call
  learning-sensitive with restarts (LSR) backdoors, that extends the
  notion of learning-sensitive (LS) backdoors to incorporate restarts
  and discuss algorithms to compute them. We mathematically prove that
  for certain class of instances minimal LSR-backdoors are
  exponentially smaller than minimal-LS backdoors.
\end{abstract}

\section{Introduction}
Modern conflict-driven clause-learning (CDCL) satisfiability (SAT)
solvers routinely solve real-world Boolean instances with millions of
variables and clauses, despite the Boolean satisfiability problem
being NP-complete and widely regarded as intractable in
general.
This has perplexed both theoreticians and solver developers alike over
the last two decades. A commonly proposed explanation is that these
solvers somehow exploit the underlying structure inherent in
industrial instances. Previous work has attempted to identify a
variety of structural parameters, such as
backdoors~\cite{Williams2003,szeider2003fixed,Dilkina2014}, community
structure modularity \cite{Ansotegui2012,newsham2014impact}, and
treewidth~\cite{Mateescu2011}. Additionally, researchers have
undertaken limited studies to correlate the size/quality measures of
these parameters with CDCL SAT solver runtime. For example, Newsham et
al.~\cite{newsham2014impact} showed that there is moderate correlation
between the runtime of CDCL SAT solvers and modularity of community
structure of industrial instances. Others have shown that certain
classes of crafted and industrial instances often have favorable
parameter values, such as small backdoor sizes
\cite{Dilkina2014,Kilby2005,Li2011,Williams2003}.


There are several reason for trying to understand why SAT solvers work as well as they do on industrial instances, and what, if any, structure they exploit. First and foremost is the scientific motivation: a common refrain about heuristics in computer science is that we rarely fully understand why and under what circumstances they work well.
Second, a deeper understanding of the relationship between the
power of SAT heuristics and the structure of instances over which they
work well can lead us to develop better SAT solvers. Finally, we hope
that this work will eventually lead to new parameterized
complexity results relevant to classes of industrial instances.

Before we get into the details of our empirical study, we briefly
discuss several principles that guided us in this work: first, we
wanted our study to be comprehensive both in terms of the parameters
as well as the large size and variety of benchmarks considered.  This
approach enabled us to compare the explanatory strengths of different
parameters in a way that previous studies could not. Also, the large
scale and variety of benchmarks allowed us to draw more general
conclusions than otherwise. Second, we were clear from the outset that
we would be agnostic to the type of correlations (strong or moderate,
positive or negative) obtained, as poor correlations are crucial in
ruling out parameters that are not explanatory. Third, parameter
values should enable us to distinguish between industrial and crafted
instances.

Although many studies of these parameters have been performed in
isolation, a comprehensive comparison has not been performed between
them. A primary reason for this is that most of these parameters are
difficult to compute -- often NP-hard -- and many take longer to
compute than solving the original formula. Further, certain parameters
such as \textit{weak} backdoor size are only applicable to satisfiable
instances \cite{Williams2003}. Hence, such parameters have often been
evaluated on incomparable benchmark sets, making a proper comparison
between them difficult. We correct this issue in our study by focusing
on instances found in previous SAT competitions, specifically from the
Application, Crafted, and Agile tracks \cite{satcomps}. These
instances are used to evaluate state-of-the-art SAT solvers on a
yearly basis. Application instances are derived from a wide variety of
sources and can be considered a small sample of the types of SAT
instances found in practice, such as from verification
domains. Crafted instances mostly contain encodings of
combinatorial/mathematical properties, such as the pigeon-hole
principle or pebbling formulas. While many of these instances are much
smaller than industrial instances, they are often very hard for CDCL
solvers. The Agile track evaluates solvers on bit-blasted
quantifier-free bit-vector instances generated from the whitebox fuzz
tester SAGE~\cite{godefroid2008automated}. In total, we consider
approximately 1200 application instances, 800 crafted instances, and
5000 Agile instances.

\vspace{0.2cm}
\noindent{\bf Contributions.} We make the following four key
contributions in this paper:

\begin{enumerate}
\item We perform a large scale evaluation of several structural
  parameters on approximately 7000 SAT instances obtained from a wide variety
  of benchmark suites, and relate them to CDCL solver performance. We
  show moderate correlations to solving time for certain combinations
  of features, and very high correlations for the Agile benchmark. We
  further show that application and Agile instances have significantly
  better structural parameter values compared to crafted instances. To
  the best of our knowledge, this is first such comprehensive
  study. (Refer Section~\ref{sec_sat_analysis})
  
\item We introduce a new structural parameter, which we call learning-sensitive with restarts (LSR) backdoors, and describe a novel
  algorithm for computing upper bounds on minimum LSR-backdoor sizes, using the concept of
  clause absorption from \cite{atserias2011clause}. The LSR-backdoor
  concept naturally extends the idea of learning-sensitive (LS)
  backdoors introduced by Dilkina et al.~\cite{Dilkina2009a} by taking
  into account restarts, a key feature of CDCL SAT solvers. (Refer
  Background Section~\ref{sec:background} and
  Section~\ref{sec_lsr_algorithms})

\item We show how these structural parameters can be used as a lens to
  compare various solving heuristics, with a current focus on restart
  policies.  For example, we show that the ``Always Restart'' policy
  not only produces the smallest backdoors (with respect to the algorithm we use from Section \ref{sec_lsr_algorithms}), but also has the fastest
  runtime for a class of instances. We hope that our work can be used
  as a guide by theorists and SAT solver developers to rule out the
  study of certain types of parameters (and rule in new types of
  parameters) for understanding the efficiency of CDCL SAT
  solvers. (Refer Section~\ref{sec_sat_analysis})

\item We mathematically prove that minimal LSR-backdoors can be
  exponentially smaller than LS-backdoors for certain formulas. (Refer
  Section~\ref{sec_lsr_theory})
\end{enumerate}

\section{Background}
\label{sec:background}
\textbf{CDCL SAT Solvers.}  We assume basic familiarity with the
satisfiability problem, CDCL solvers and the standard notation used by
solver developers and complexity theorists. For an overview we refer
to \cite{biere2009handbook}. We assume that Boolean formulas are given
in conjunctive normal form (CNF). For a formula $F$, we denote its variables
as \textit{vars(F)}. A \textit{model} refers to a complete satisfying
assignment to a formula.  The \textit{trail} refers to the sequence of
variable assignments, in the order they have been assigned, at any
given point in time during the run of a solver. Learnt clauses are
derived by analyzing the \textit{conflict analysis graph} (CAG), which
represents which decisions and propagations that led to a conflict. We
assume the first unique implication point (1st-UIP) clause learning
scheme throughout this paper, which is the most common in practice
\cite{moskewicz2001chaff}. The learnt clause (aka conflict clause)
defines a cut in the CAG; we denote the subgraph on the side of the
cut which contains the conflict node as the \textit{conflict side},
and the other side as the \textit{reason side}.  For a set of clauses
$\Delta$ and a literal $l$, $\Delta \vdash_1 l$ denotes that unit
propagation can derive $l$ from $\Delta$.

\textbf{Backdoors and Backbones.}  Backdoors are defined with respect
to \textit{subsolvers}, which are algorithms that can solve certain
class of SAT instances in polynomial-time. Example subsolvers
include the unit propagation (UP) algorithm
\cite{Dilkina2014,Williams2003}, which is also what we focus on as
it is a standard
subroutine implemented in CDCL SAT solvers.  Given a partial
assignment $\alpha: B \rightarrow \{0,1\}$, the simplification of $F$
with respect to $\alpha$, denoted $F[\alpha]$, removes all clauses
that are satisfied by $\alpha$, and removes any literals that are
falsified by $\alpha$ in the remaining clauses.  A \textit{strong
  backdoor} of a formula $F$ is a set of variables $B$ such that for
every assignment $\alpha$ to $B$, $F[\alpha]$ can be determined to be
satisfiable or unsatisfiable by the UP subsolver \cite{Williams2003}.
A set of variables $B$ is a \textit{weak backdoor} with respect to a
subsolver $S$ if there exists an assignment $\alpha$ to $B$ such that
the UP subsolver determines the formula to be satisfiable.  Backdoors
were further extended to allow clause-learning to occur while
exploring the search space of the backdoor:

\begin{definition}[Learning-sensitive (LS) backdoor \cite{Dilkina2009}]
	A set of variables $B \subseteq vars(F)$ is an LS-backdoor of
        a formula $F$ with respect to a subsolver $S$ if there exists
        a search-tree exploration order such that a CDCL SAT solver
        without restarts, branching only on $B$ and learning clauses
        at the leaves of the tree with subsolver $S$, either finds a
        model for $F$ or proves that $F$ is unsatisfiable.
\end{definition}

The backbone of a SAT instance is the maximal set of variables such that all variables in the set have the same polarity in every model \cite{monasson1999determining}.
Note that weak backdoors and
backbones are typically only defined over satisfiable
instances. Further, while the backbone of an instance is unique, many
backdoors may exist; we typically try to find the smallest backdoors
possible.

\textbf{Learning-sensitive with Restarts (LSR) Backdoors.}
\label{bg_lsr}
We introduce the concept of an LSR-backdoor here.  Section
\ref{sec_lsr_algorithms} formalizes our approach.

\begin{definition}[Learning-sensitive with restarts (LSR) backdoor]
	A set of variables $B \subseteq vars(F)$ is an LSR-backdoor of
        the formula $F$ with respect to a subsolver $S$ if there
        exists a search-tree exploration order such that a CDCL SAT
        solver (with restarts), branching only on $B$ and learning
        clauses at the leaves of the tree with subsolver $S$, either
        finds a model for $F$ or proves that $F$ is unsatisfiable.
\end{definition}

By allowing restarts, the solver may learn clauses from different
parts of the search-tree of $B$, which would otherwise not be
accessible without restarts. 
In Section
\ref{sec_lsr_algorithms}, we demonstrate an approach to computing upper bounds on minimal LSR-backdoor sizes using algorithms for \textit{clause absorption} from
\cite{atserias2011clause}, which intrinsically relies upon restarts.

\textbf{Graph Parameters.}  We refer to treewidth~\cite{robertson1984graph} and
community structure~\cite{Ansotegui2012} as graph parameters. All graphs parameters
are computed over the \textit{variable incidence graph (VIG)} of a CNF
formula $F$ \cite{Ansotegui2012}. There exists a vertex for every
variable in $F$, and edges between vertices if their corresponding
variables appear in clauses together (weighted according to clause
size).  Intuitively, the community structure of a graph is a partition
of the vertices into communities such that there are more
intra-community edges than inter-community
edges. The \textit{modularity} or \textit{Q} value intuitively denotes
how easily separable the communities of the graph are. The \textit{Q}
value ranges from $[-1/2, 1)$, where values near $1$ means the communities are
highly 
separable and the graph intuitively has a \textit{better} community
structure.  \textit{Treewidth} intuitively measures how much a graph
resembles a tree. Actual trees have treewidth 1. Further details can
be found in \cite{robertson1984graph}.

\section{Related Work}
\label{sec-relatedwork}

\textbf{Backdoor-related Parameters.} Traditional weak and strong
backdoors for both SAT and CSP were introduced by Williams et al.~
\cite{Williams2003}.  Kilby et al. introduced a local search algorithm
for computing weak backdoors~\cite{Kilby2005}, and showed that random
SAT instances have medium sized backdoors (of roughly 50\% of the
variables), and that the size of weak backdoors did not correlate
strongly with solving time. Li et al. introduced an improved Tabu
local search heuristic for weak backdoors \cite{Li2011}. They
demonstrated that many industrial satisfiable instances from SAT
competitions have very small weak backdoors, often around 1\% of the
variables. The size of backdoors with respect to subsolvers different
from UP was considered in \cite{Dilkina2014,Li2011}. Monasson et
al. introduced backbones to study random 3-SAT instances
\cite{monasson1999determining}. Janota et
al. \cite{janota2015algorithms} introduced and empirically evaluated
several algorithms for computing backbones. Several extensions of
traditional strong and weak backdoors have been
proposed. Learning-sensitive (LS) backdoors also consider the
assignment tree of backdoor variables, but additionally allow clause
learning to occur while traversing the tree, which may yield
exponentially smaller backdoors than strong backdoors
\cite{Dilkina2009,Dilkina2009a}.

\textbf{Graph Abstraction Parameters.}  Mateescu computed lower and
upper bounds on the treewidth of large application formulas
\cite{Mateescu2011}.  Ans\'otegui et al. introduced community
structure abstractions of SAT formulas, and demonstrated that
industrial instances tend to have much better structure than other
classes, such as random \cite{Ansotegui2012}. It has also been shown
that community-based features are useful for classifying industrial
instances into subclasses (which distinguish types of industrial
instances in the SAT competition) \cite{jordi2015classification}.
Community-based parameters have also recently been shown to be one of
the best predictors for SAT solving performance
\cite{newsham2014impact}.

\begin{table}
	
\begin{center}
	\begin{tabular}{| l | m{5.5em} | c | m{19em} |}
		\hline
		Type & Benchmarks & Unsat? & Tool Description \\ \hline
		Weak \cite{Kilby2005,Li2011,Chaharmir2002} & 3SAT, GC, SR & No & Perform Tabu-based  local search to minimize the number of decisions in the final model.  \\ \hline
		LS \cite{Dilkina2009a,Dilkina2014} & LP  & Yes & Run a clause-learning solver, recording all decisions, which constitutes a backdoor.  \\ \hline
		Backbones \cite{janota2015algorithms,Kilby2005,Chaharmir2002} & 3SAT, GC, Comps & No & Repeated SAT calls with UNSAT-core based optimizations.
		\\ \hline
		Treewidth \cite{Liang2015,Mateescu2011} & C09, FM & Yes & Heuristically compute residual graph $G$. The max-clique of $G$ is an upper-bound. \\ \hline
		Modularity \cite{Ansotegui2012,newsham2014impact} & Comps & Yes & The \textit{Louvain method} \cite{blondel2008fast} -- greedily join communities to improve modularity. \\ \hline
	\end{tabular}
\end{center}
\caption{Previously studied benchmarks for each considered parameter, as well as description of tools used to compute them. The ``Unsat?'' column indicates if the parameter is defined on unsatisfiable instances. Abbrevations: 3SAT -- random 3-SAT; GC -- graph coloring; LP -- logistics planning; SR -- SAT Race 2008;
	C09 -- SAT competition 2009; Comps -- 2009-2014 SAT competitions; FM -- feature models. }
\label{tab-relatedwork}
 
	\begin{center}
\begin{tabular}{|l|c|c|c|c|c|c|}
\hline
Benchmark &  Instances &   LSR &  Weak &  Cmty &  Bones &    TW \\
\hline
Application &       1238 &   420 &   306 &   984 &    218 &  1181 \\
Crafted     &        753 &   327 &   195 &   613 &    154 &   753 \\
Random      &        126 &   123 &    76 &   126 &     59 &   126 \\
Agile       &       4968 &  2828 &   464 &  4968 &    208 &  4968 \\ \hline
Total       &       7085 &  3698 &  1041 &  6691 &    639 &  7028 \\
\hline
\end{tabular}
\end{center}
\caption{ The number of instances for which we were able to successfully compute each parameter. ``Cmty'' refers to the community parameters; ``TW'' denotes the treewidth upper bound; ``Bones'' denotes backbone size. }
\label{tab_data_summary}

		\begin{center}
\scalebox{0.93}{
\begin{tabular}{|l|c|c|c|c|}
\hline
\textbf{Feature Set} & \textbf{Application} & \textbf{Crafted} & \textbf{Random} & \textbf{Agile}\\ \hline

$V\oplus{}C\oplus{}C/V$ & 0.03 (1237) & 0.04 (753) & 0.04 (126) & 0.84 (4968)\\

$V\oplus{}C\oplus{}Cmtys\oplus{}Q$ & 0.06 (982) & 0.22 (613) & 0.17 (126) & 0.86 (4968)\\

$V\oplus{}C\oplus{}LSR\oplus{}LSR/V$ & 0.14 (420) & 0.26 (327) & 0.26 (123) & 0.87 (2828)\\

$V\oplus{}C\oplus{}\#Min\_Weak\oplus{}Weak$ & 0.04 (299) & 0.11 (195) & 0.08 (76) & 0.54 (464)\\

$V\oplus{}C\oplus{}Bones\oplus{}Bones/V$ & 0.18 (218) & 0.39 (154) & 0.04 (59) & 0.39 (208)\\

$V\oplus{}C\oplus{}TW\oplus{}TW/V$ & 0.05 (1180) & 0.07 (753) & 0.11 (126) & 0.91 (4968)\\

\hline
$Q\oplus{}C/V\oplus{}LSR/V\oplus{}Q/Cmtys\oplus{}C$ &\textbf{0.29 (420)} & 0.34 (327) & 0.13 (123) & 0.90 (2828)\\

$TW/V\oplus{}Q\oplus{}Cmtys\oplus{}TW\oplus{}LSR/V$ & 0.12 (420) & \textbf{0.57 (327)} & 0.08 (123) & 0.92 (2828)\\

$Q/Cmtys\oplus{}LSR/V\oplus{}C\oplus{}LSR\oplus{}Q$ & 0.22 (420) & 0.35 (327) & \textbf{0.45 (123)} & 0.89 (2828)\\

$Cmtys\oplus{}TW/V\oplus{}C/V\oplus{}TW\oplus{}Q$ & 0.18 (420) & 0.29 (327) & 0.04 (123) & \textbf{0.93 (2828)}\\

\hline
\end{tabular}
}
\end{center}
\caption{ Adjusted R$^2$ values for the given features, compared to log of MapleCOMSPS' solving time. The number in parentheses indicates the number of instances that were considered in each case. The lower section considers heterogeneous sets of features across different parameter types.}
\label{tab-regressions}

	\begin{center}
\begin{tabular}{|l|c|c|c|c|c|}
\hline
Benchmark &        LSR/V &       Weak/V &            Q &      Bones/V &         TW/V \\
\hline
Agile       &  0.18 (0.13) &  0.01 (0.01) &  0.82 (0.07) &  0.17 (0.11) &  0.16 (0.08) \\
Application &  0.35 (0.34) &  0.03 (0.05) &  0.75 (0.19) &  0.64 (0.38) &  0.32 (0.22) \\
Crafted     &  0.58 (0.35) &  0.08 (0.11) &  0.58 (0.24) &  0.39 (0.41) &  0.44 (0.29) \\
Random      &  0.64 (0.32) &  0.11 (0.10) &  0.14 (0.10) &  0.47 (0.40) &  0.82 (0.12) \\
\hline
\end{tabular}
\end{center}
\caption{ Mean (std. dev.) of several parameter values. }
\label{tab-meanstd}

\end{table}


\suppress{ \textbf{Proof Measures.} J{\"a}rvisalo et al. studied how proof
  measures, such as proof length, width, and space, are related to
  CDCL performance \cite{jarvisalo2012relating}. They showed that
  proof size correlated well with solving time, over a set of crafted
  formulas with fixed width and length.  }

Other work such as SatZilla \cite{xu2008satzilla} focus on large sets
of easy-to-compute parameters that can be used to quickly predict the
runtime of SAT solvers. In this paper, our focus is on parameters
that, if sufficiently favorable, offer provable parameterized
complexity-theoretic guarantees of worst-case runtime
\cite{downey2013fundamentals}. The study of structural parameters of
SAT instances was inspired by the work on clause-variable ratio and
the phase transition phenomenon observed for randomly-generated SAT
instances in the late
1990's~\cite{coarfa2000random,monasson1999determining,selman1996generating}.
 


Table \ref{tab-relatedwork} lists previous results on empirically
computing several parameters and correlating them with SAT solving
time.  While weak backdoors, backbones, and treewidth have been
evaluated on some industrial instances from the SAT competitions, only
modularity has been evaluated across a wide range of instances. Some
benchmarks, such as the random 3-SAT and graph coloring instances
considered in \cite{Kilby2005}, are too small/easy for modern day CDCL
solvers to perform meaningful analysis. Additionally, the benchmarks
used in previous works to evaluate each parameter are mostly disjoint,
making comparisons across the data difficult.

\section{Analysis of Structural SAT Parameters}
\label{sec_sat_analysis} 

Our first set of experiments investigate the relationship between
structural parameters and CDCL performance. While we would like to
evaluate all parameters considered in Section \ref{sec-relatedwork},
we focus on weak backdoors, backbones, community structure, treewidth,
and LSR-backdoors. We note that obtaining any non-trivial upper bound
on the size of the strong backdoor seems infeasible at this time.

\suppress{
\begin{table}[t]
	\centering
	\begin{tabular}{|l|rrrrrr|}
		\hline
		{} &  \#Instances &   LSR &  Weak &  Cmty  & Backbones &  maplecomsps \\ \hline
		agile    &       5000 &  1829 &   377 &  5000 &  &      2543   \\
		application &       1048 &   482 &   257 &   888 & &         1045  \\
		crafted  &        783 &   300 &     0 &   726 &  &           783 \\
		random   &        126 &   106 &     0 &   126 & &          126 \\ 
		Total    &       6957 &  2717 &   634 &  6740 &  &        4497 \\ \hline
	\end{tabular}
	\caption{\todo{column for sat}Depicts the number of instances
          in each benchmark, as well as the number of instances we
          were able to successfully compute each metric/time. ``Cmty''
          refers community metrics.}
	\label{tab-datasummary}
\end{table}
}

\suppress{
\begin{table}
	\centering
	\begin{tabular}{|l|ccccc|}
		\hline
		{} &   LSR &  Weak+keys &  Cmty  & Backbones & maplecomsps \\ \hline
		agile    &  (6GB, 300s) &   (6GB, 1h) &  (6GB, 1h) & (6GB, 24h) & (6GB, 60s)   \\
		app/crafted/random &  (6GB, 3h) &   (6GB, 24h) &  (6GB, 2h) & (6GB, 24h) & (6GB, 5000s)   \\ \hline
	\end{tabular}
	\caption{\todo{fold into text not worth the space}Allocated resources for each experiment.}
	\label{tab-resources}
\end{table}
}

\textbf{Experimental Setup, Tools, and Benchmarks.}  We use off the
shelf tools to compute weak backdoors \cite{Li2011}, community
structure and modularity \cite{newsham2014impact}, backbones
\cite{janota2015algorithms}, and treewidth \cite{Mateescu2011}. Their
approaches are briefly described in Table \ref{tab-relatedwork}. Due
to the difficulty of exactly computing these parameters, with the
exception of backbones, the algorithms used in previous work (and our
experiments) do not find optimal solutions, e.g., the output may be an
upper-bound on the size of the minimum backdoor. We compute
LSR-backdoors using a tool we developed called \textit{LaSeR}, which
computes an upper-bound on the size of the minimal LSR-backdoor. The
tool is built on top of the MapleSat SAT solver
\cite{liang2016learning},  an extension of MiniSat
\cite{een2003extensible}.  We describe the \textit{LaSeR} algorithm in
Section \ref{sec_lsr_algorithms}. We use \textit{MapleCOMSPS}, the
2016 SAT competition main track winner as our reference solver for
solver runtime.

Table \ref{tab_data_summary} shows the data sources for our
experiments. We include all instances from the Application and Crafted
tracks of the SAT competitions from 2009 to 2014, as well as the 2016
Agile track. We additionally included a small set of random instances
as a baseline. As the random instances from recent SAT competitions
are too difficult for CDCL solvers, we include a set of instances from
older competitions. We pre-selected all random instances from the 2007
and 2009 SAT competitions that could be solved in under 5 minutes by
MapleCOMSPS. All instances were simplified using MapleCOMSPS'
preprocessor before computing the parameters. The preprocessing time
was not included in solving time.
 
Experiments were run on an Azure cluster, where each node contained
two 3.1 GHz processors and 14 GB of RAM. Each experiment was limited
to 6 GB. For the Application, Crafted, and Random instances, we
allotted 5000 seconds for MapleCOMSPS solving (the same as used in the
SAT competition), 24 hours for backbone and weak backdoor computation,
2 hours for community structure computation, and 3 hours for LSR
computation. For the Agile instances, we allowed 60 seconds for
MapleCOMSPS solving and 300 seconds for LSR computation; the remaining
Agile parameter computations had the same cutoff as Application. Due
to the difficulty of computing these parameters, we do not obtain
values for all instances due to time or memory limits being
exceeded.  Several longer running experiments were run on the SHARCNET
Orca cluster \cite{sharcnet}. Nodes contain cores between 2.2GHz and
2.7GHz.

\textbf{Structural Parameters and Solver Runtime Correlation:} The
first research question we posed is the following: \textit{Do
  parameter values correlate with solving time? In particular, can we
  build significantly stronger regression models by incorporating
  combinations of these features?} To address this, we
construct ridge regression models from subsets of features related to
these parameters. We used ridge regression, as opposed to linear
regression, in order to penalize multi-collinear features in the
data. We consider the following ``base'' features: number of variables
(V), number of clauses (C), number of communities (Cmtys), modularity
(Q), weak backdoor size (Weak), the number of minimal weak backdoors
computed (\#Min\_Weak), LSR-backdoor size (LSR), treewidth upper-bound
(TW), and backbone size (Bones). For each $P \in \{$\textit{C, Cmtys,
  Weak, LSR, TW, Bones}$\}$ we include its ratio with respect to $V$
as $P/V$. We also include the ratio feature $Q/Cmtys$, as used in
\cite{newsham2014impact}. All features are normalized to have mean 0
and standard deviation 1. For a given subset of these features under
consideration, we use the ``$\oplus$'' symbol to indicate that our
regression model contains these base features, as well as all
higher-order combinations of them (combined by multiplication). For
example, $V \oplus C$ contains four features: $V$, $C$, and $V \cdot
C$, plus one ``intercept'' feature. Our dependent variable is the log
of runtime of the MapleCOMSPS solver.

In Table \ref{tab-regressions}, we first consider sets of
\textit{homogeneous} features with respect to a given parameter, e.g.,
only weak backdoor features, or only community structure based
features, along with $V$ and $C$ as baseline features. Each cell
reports the adjusted $R^2$ value of the regression, as well as the
number of instances considered in each case (which corresponds to the
number of instances for which we have data for each feature in the
regression). It is important to note that since different subsets of SAT
formulas are used for each regression (since our dataset is
incomplete), we should not directly compare the cells in the top
section of the table. Nonetheless, the results do give some indication
if each parameter relates to solving time.

In order to show that combinations of these features can produce
stronger regression models, in the bottom half of Table
\ref{tab-regressions}, we consider all instances for which we have
LSR, treewidth, and community structure data. We exclude backbones and
weak backdoors in this case, as it limits our study to SAT instances
and greatly reduces the number of datapoints.  We considered all
subsets of base features of size 5 (e.g. $V \oplus LSR/V \oplus LSR
\oplus Q \oplus TW$), and report the best model for each benchmark,
according to adjusted $R^2$ (i.e. the bolded data along the
diagonal). This results in notably stronger correlations than with any
of the homogeneous features sets. Although we report our results with
five base features (whereas most homogeneous models only used four),
similar results appear if we only use four base features.  We also
note that $R^2$ values results to be higher for the Agile instances,
as compared to application and crafted instances. This is somewhat
expected, as the set of instances are all derived from the SAGE
whitebox fuzzer~\cite{godefroid2008automated}, as compared to our other benchmarks
which come from a heterogeneous set of sources.

For each row corresponding to the heterogeneous feature sets, the base
features are ordered according to the \textit{confidence level}
(corresponding to p-values), that the feature is significant to the
regression (highest first), with respect to the model used to produce
the bold data along the diagonal (i.e. the best model for each
benchmark). Confidence values are measured as a percentage; for
brevity, we consider values over 99\% as very significant, values
between 99\% and 95\% are significant, and values below 95\% are
insignificant. For application instances, \textit{Q}, \textit{C/V},
and \textit{LSR/V} are all very significant, \textit{Q/Cmtys} was
significant, and $C$ was not significant.  For crafted instances,
\textit{TW/V} and \textit{Q} were very significant, but the other base
features were insignficant.  No features are significant for the
random track, indicating that the $R^2$ value is likely spurious,
partly due to the small size of the benchmark. For the Agile
benchmark, all five features are very significant. In each model,
several higher-order features are also reported as significant,
including several where the base feature is not considered
significant.

We also remark that previous work showed notably higher $R^2$ values
for community-based features \cite{newsham2014impact}. There are
several significant differences between our experiments. First, our
instances are pre-simplified before computing community
structure. Their experiments grouped all Application, Crafted, and
Random into a single benchmark, whereas ours are split.



\suppress{
\begin{table}[t]
	\centering
	\begin{tabular}{|l|rrr|}
		\hline
		benchmark &       lvr &      wvr &         Q \\ \hline
		agile     &  0.224 (0.135) &  0.004 (0.002) &  0.870 (0.054) \\
		app       &  0.320 (0.319) &  0.015 (0.024) &  0.789 (0.183) \\
		crafted   &  0.573 (0.268) &      N/A &  0.630 (0.250)\\
		random    &  0.476 (0.152) &      N/A &  0.147 (0.104)\\
		\hline
	\end{tabular}
	\caption{Mean (std. dev.) values for each parameter.}
	\label{tab-meanstd}
\end{table}
}

\textbf{Structural Parameters for Industrial vs. Crafted Instances:}
The research question we posed here is the following: \textit{Do
  real-world SAT instances have significantly more favorable parameter
  values (e.g. smaller backdoors or higher modularity), when compared
  to crafted or random instances?}  A positive result, such that
instances from application domains (including Agile) have better
structural values, would support the hypothesis that such structure
may relate to the efficiency of CDCL solvers. Table \ref{tab-meanstd}
summarizes our results.  We note that while application and Agile
instances indeed appear more structured with respect to these
parameters, the application benchmark has high standard deviation
values. This could be due to the application instances
coming from a wide variety of sources.

\subsection{Using Structural Parameters to Compare Solving Heuristics}
When comparing different solvers or heuristics, the most common
approach is to run them on a benchmark to compare solving times, or in
some cases the number of conflicts during solving. However, such an
approach does not lend much insight as to \textit{why} one heuristic
is better than another, nor does it suggest any ways in which a less
performant heuristic may be improved. By contrast, in their recent
works Liang et al.~\cite{liang2016learning,liang2015understanding},
drew comparisons between various branching heuristics by comparing
their locality with respect to the community structure or the
``learning rate'' of the solver. This eventually led them to build
much better branching heuristics. We hope to do the same by using
LSR-backdoors as a lens to compare restart policies.

Starting from the baseline MapleSAT solver, we consider three restart
heuristics: 1) the default heuristic based on the Luby sequence
\cite{luby1993optimal}; 2) restarting after every conflict (always restart); and 3)
never restarting.  We test the following properties of these
heuristics, which may relate to their effect on solving performance.
First, we consider the LSR-backdoor size, as computed by the approach discussed
in Section \ref{sec_lsr_algorithms}. A run of the solver that focuses
on fewer variables can be seen as being more ``local,'' which may be
favorable in terms of runtime. When computing LSR-backdoor sizes, each learnt
clause $C$ is annotated with a set of \textit{dependency variables}
(denoted $D_C^*$ in Section \ref{sec_lsr_algorithms}), which
intuitively is a sufficient set of variables such that a fresh solver,
branching only on variables in this set, can learn $C$. Our second
measure looks at the average dependency set size for each learnt
clause.

\begin{table}[t]
\centering
\begin{center}
	\begin{tabular}{ |l|c|c|c| }
		\hline
		\textbf{Property} & \textbf{Luby} & \textbf{Always Restart} & \textbf{Never Restart} \\ \hline
		LSR Size& 0.16 (0.11) & \textbf{0.13 (0.08)} & 0.25 (0.17)\\
		Avg. Clause LSR& 0.11 (0.08) & \textbf{0.05 (0.05)} & 0.19 (0.14)\\
		Num Conflicts& 133246 (206441) & \textbf{50470 (84606)} & 256046 (347899)\\
		Solving Time (s)& 8.59 (14.26) & \textbf{6.09 (11.01)} & 18.09 (24.86)\\
		\hline
	\end{tabular}
\end{center}

\caption{Comparison of LSR measures and solving time for various restart policies on the Agile benchmark. LSR sizes are normalized by the number of variables.}
\label{tab_lens}

\end{table}

Due to space limitations, we only consider the Agile instances as this
gives us a large benchmark where all instances are from the same
domain.  The data in Table \ref{tab_lens} corresponds to the average
and standard deviation values across the benchmark. We only consider
instances where we can compute all data for each heuristic, in total
2145 instances. The always restart policies emits smaller LSR sizes,
both overall and on average per clause, and, somewhat surprisingly,
always restarting outperforms the more standard Luby policy in this
context.  Note that we do not expect this result to hold across all
benchmarks, as it has been shown that different restart
policies are favorable for different types of instances
\cite{biere2015evaluating}.  However, given the success of always
restarting here, results such as ours may promote techniques to
improve rapidly restarting solvers such as in~\cite{ramos2011between}.

\section{Computing LSR-Backdoors}
\label{sec_lsr_algorithms}

Dilkina et al.~\cite{Dilkina2009,Dilkina2014} incorporated clause
learning into the concept of strong backdoors 
by introducing LS-backdoors, and additionally described an approach for
empirically computing upper bounds on minimal LS-backdoors.
We refer the reader to our [anonymized] extended technical report for complete proofs \cite{extended_paper_link}.

We propose a new concept called learning-sensitive with restarts (LSR)
backdoors and an approach that takes advantage of allowing restarts
which can often greatly reduce the number of decisions necessary to
construct such a backdoor especially if many ``unnecessary'' clauses
are derived during solving. Our key insight is that, as stated in
\cite{ansotegui2015using,oh2015between}, most learnt clauses are 
ultimately not used to determine SAT or UNSAT, and therefore we only need
to consider variables required to derive such ``useful'' clauses. Our
result shows that, for an unsatisfiable formula, the set of variables
within the set of learnt clauses in the UNSAT proof constitutes an
LSR-backdoor. The result for satisfiable formulas shows that the set
of decision variables in the final trail of the solver, along with the
variables in certain learnt clauses, constitute an
LSR-backdoor. Before describing result, we first recall the properties
of \textit{absorption}, \textit{1-empowerment}, and
\textit{1-provability}, which were initially used to demonstrate that
CDCL can simulate general resolution within some polynomial-size
bound:

\begin{definition}[Absorption \cite{atserias2011clause}]
  Let $\Delta$ be a set of clauses, let $C$ be a non-empty clause and
  let $x^\alpha$ be a literal in $C$. Then $\Delta$ absorbs $C$ at
  $x^\alpha$ if every non-conflicting state of the solver that
  falsifies $C \setminus \{x^\alpha\}$ assigns $x$ to $\alpha$. If
  $\Delta$ absorbs $C$ at every literal, then $\Delta$ absorbs $C$.
\end{definition}
	
The intuition behind absorbed clauses is that adding an already
absorbed clause $C$ to $\Delta$ is in some sense redundant, since any
unit propagation that could have been realized with $C$ is already
realized by clauses in $\Delta$.
	
\begin{definition}[1-Empowerment \cite{pipatsrisawat2008new}]
  Let $\alpha \Rightarrow l$ be a clause where $l$ is some literal in
  the clause and $\alpha$ is a conjunction of literals. The clause is
  1-empowering with respect to a set of clauses $\Delta$ if:
  \begin{enumerate}
  \item $\Delta \models (\alpha \Rightarrow l)$: the clause is implied
    by $\Delta$.
  \item $\Delta \wedge \alpha$ does not result in a conflict
    detectable by unit propagation.
  \item $\Delta \wedge \alpha \not \vdash_{1} l$: unit propagation
    cannot derive $l$ after asserting the literals in $\alpha$.
  \end{enumerate} 
\end{definition}
	
\begin{definition}[1-Provability \cite{pipatsrisawat2009power}]
  Given a set of clauses $\Delta$, a clause $C$ is 1-provable with
  respect to $\Delta$ \textit{iff} $\Delta \wedge \neg C \vdash_{1}
  false$.
\end{definition}
		
An important note is that every learnt clause is both 1-empowering and
1-provable, and therefore not absorbed, at the moment it is derived by
a CDCL solver (i.e., before being added to $\Delta$)
\cite{pipatsrisawat2008new,pipatsrisawat2009power}.

\begin{lemma}
  \label{absorb-lemma}		 
  Let $\Delta$ be a set of clauses and suppose that C is a
  1-empowering and 1-provable clause with respect to $\Delta$. Then
  there exists a sequence $\sigma$ of decisions and restarts
  containing only variables from C such that $\Delta$ and the set of
  learned clauses obtained from applying $\sigma$ absorbs C.
\end{lemma}
\begin{proof} 
  The proof follows directly from the construction of such a decision
  sequence in the proof of Proposition 2 of
  \cite{pipatsrisawat2009power}.
\end{proof}

Our result additionally makes use of the following notation.  Let $F$
be a formula and $S$ be a CDCL solver. We denote the full set of
learnt clauses derived during solving as $S_L$. For every
conflicting state, let $C^\prime$ denote the clause that will be
learned through conflict analysis.  We let $R_{C^\prime}$ be the set
of clauses on the conflict side of the implication graph used to
derive $C^\prime$ where $R_{C^\prime}^* = R_{C^\prime} \cup \bigcup_{C
  \in R_{C^\prime}} R_C^*$ recursively defines the set of clauses
needed to derive $C^\prime$ (where $R_{original\_clause}^* =
\emptyset$). For every learnt clause we define $D_{C^\prime}^* =
vars(C^\prime) \cup \bigcup_{C \in R_{C^\prime}^*} D_C^*$, where
$D_{original\_clause}^* = \emptyset$, as the set of variables in the
clause itself as well as any learnt clause used in the derivation of
the clause (recursively). Intuitively, $D_{C^\prime}^*$ is a
sufficient set of \textit{dependency variables}, such that a fresh SAT
solver can absorb $C^\prime$ by only branching on variables in the
set. For a set of clauses $\Delta$, we let $R_\Delta^* = \bigcup_{C\in
  \Delta} R_C^*$ and $D_\Delta^* = \bigcup_{C \in \Delta} D_C^*$.

\begin{lemma}
Let $S$ be a CDCL solver used to determine the satisfiability of some
formula $F$. Let $\Delta \subseteq S_L$ be a set of clauses learned while solving $F$. Then a fresh solver $S^\prime$ can absorb
all clauses in $\Delta$ by only branching on the variables in
$D_{\Delta}^*$.
\label{lem_absorb_clauses}
\end{lemma}

\begin{sketch}
  Let $seq(R_\Delta^*)=\langle C_1, C_2, \ldots C_n \rangle$ be the
  sequence over $R_\Delta^*$ in the order that the original solver $S$
  derived the clauses, and suppose we have already absorbed the first
  $k-1$ clauses by only branching upon $D_{\{C_1, \ldots, C_{k-1}
    \}}^*$. Then, in particular, the clauses in $R_{C_k}$ have been
  absorbed, so $C_k$ must be 1-provable. If $C_k$ is not 1-empowering,
  then it is absorbed and we are done. If $C_k$ is 1-empowering, we
  can invoke Lemma \ref{lem_absorb_clauses} to absorb $C_k$ by only
  branching on variables in $C_k$, and $vars(C_k) \subseteq
  D_\Delta^*$ by construction.
\end{sketch}

\begin{theorem}[LSR Computation, SAT case]
  \label{thm_lsr_sat}
  Let $S$ be a CDCL solver, $F$ be a satisfiable formula, and $T$ be
  the final trail of the solver immediately before returning SAT,
  which is composed of a sequence of decision variables $T_D$ and
  propagated variables $T_P$. For each $p\in T_P$, let the clause used
  to unit propagate $p$ be $l_p$ and the full set of such clauses be
  $L_P$.  Then $B = T_D \cup D_{L_p}^*$ constitutes an LSR-backdoor
  for $F$.
\end{theorem}
\begin{sketch}
  Using Lemma \ref{lem_absorb_clauses}, we first absorb all clauses in
  $L_P$ by branching on $D_{L_P}^*$.  We can then restart the solver
  to clear the trail, and branch on the variables in $T_D$, using the
  same order and polarity as the final trail of $S$. Since we have
  absorbed each $l_p$, every $p$ will be propagated.
\end{sketch}
	
	
\begin{theorem}[LSR Computation, UNSAT case]
  \label{thm_lsr_unsat}
  Let $S$ be a CDCL solver, $F$ be an unsatisfiable formula, and
  $\Delta \subseteq S_L$ be the set of learnt clauses used to derive
  the final conflict.  Then $D_{\Delta}^*$ constitutes an LSR-backdoor
  for $F$.
\end{theorem}
\begin{proof}
  The result follows similarly to the satisfiable
  case. We learn all clauses relevant to the proof using
  Lemma \ref{lem_absorb_clauses}, which then allows unit
  propagation to derive UNSAT.
\end{proof}

We make some observations about our approach. First, our approach can
be easily lifted to any learning scheme other than the 1st-UIP scheme
that is currently the most widely used one.  Second, the set of
variables that constitute an LSR-backdoor may be disjoint from
the set of decisions made by the solver.  Third, the above approach
depends on the ability to restart, and therefore cannot be used to
compute LS-backdoors. In particular, the construction of the decision
sequence for Lemma \ref{absorb-lemma}, as described in
\cite{pipatsrisawat2009power}, requires restarting after every
conflict.  As an additional remark of practical importance, modern
CDCL solvers often perform clause minimization to shrink a newly
learnt clause before adding it to the clause database
\cite{sorensson2009minimizing}, which can have a significant impact on
performance. Intuitively, this procedure reduces the clause by finding
dependencies among its literals.  In order to allow clause
minimization in our experiments, for each clause $C$ we include all
clauses used by the minimizer in our set $R_{C}$.

For our empirical results, we modified an off-the-shelf solver
MapleSat \cite{liang2016learning}, by annotating each learnt clause $C^\prime$
with $D_{C^\prime}^*$. Note that we do not need to explicitly record
the set $R_{C^\prime}^*$ at any time.  As in the LS-backdoor
experiments in \cite{Dilkina2009}, different LSR-backdoors can be
obtained by randomizing the branching heuristic and polarity
selection. However, given the size and number of instances considered
here, we only perform one run per instance.

To ensure that our output is indeed an LSR-backdoor, we implemented a
verifier that works in three phases. First, we compute an LSR-backdoor
$B$ as above. Second, we re-run the solver, and record every learnt
clause $C$ such that $D_C^* \subseteq B$. We then run a final solver
with a modified branching heuristic, that iterates through the
sequence of learnt clauses from phase 2, absorbing each as described
in Lemma \ref{lem_absorb_clauses} (first checking that the clause is
either absorbed or 1-provable upon being reached in the sequence). We
ensure that the solver is in a final state by the end of the sequence.

\section{Separating LS and LSR-Backdoors}
\label{sec_lsr_theory}
In this section we prove that for certain kinds of formulas the
minimal LSR-backdoors are exponentially smaller than the minimal
LS-backdoors under the assumption that the learning scheme is 1st-UIP
and that the CDCL solver is only allowed to backtrack (and not
backjump). In \cite{Dilkina2009a}, the authors demonstrate that
LS-backdoors may be exponentially smaller than strong backdoors with
1st-UIP learning scheme but without restarts.

Let $n$ be a positive integer and let $X = \set{x_1, x_2, \ldots, x_n}$ be a set of Boolean
variables.  For any Boolean variable $x$, let $x^1$ denote the
positive literal $x$ and $x^{0}$ denote the negative literal $\neg x$.
For any assignment $\alpha \in \set{0,1}^n$ let $C_\alpha =
x_1^{1-\alpha_1} \vee x_2^{1-\alpha_2} \vee \ldots \vee
x_n^{1-\alpha_n}$ denote the clause on $x$ variables which is uniquely
falsified by the assignment $\alpha$.

Our family of formulas will be defined using the following template.  Let
${\cal O}$ be any total ordering of $\set{0,1}^n$; we write $<_{\cal O}$ to
denote the relation induced by this ordering.  The formula is defined
on the variables $x_1, x_2, \ldots, x_n$, and also three auxiliary
sets of variables $\set{q_\alpha}_{\alpha \in \set{0,1}^n},
\set{a_\alpha}_{\alpha \in \set{0,1}^n}, \set{b_\alpha}_{\alpha \in
  \set{0,1}^n}$.  Given an ordering ${\cal O}$ we define the
formula
\[{\cal F}_{\cal O} = \bigwedge_{\alpha \in \set{0,1}^n} (C_\alpha \vee \bigvee_{\alpha' \leq_{\cal O} \alpha} \neg q_{\alpha'}) \wedge (q_\alpha \vee a_\alpha) \wedge (q_\alpha \vee b_\alpha) \wedge (q_\alpha \vee \neg a_\alpha \vee \neg b_\alpha).\]
This family was introduced by Dilkina et al. \cite{Dilkina2009a}, where the formula using the lexicographic ordering
provides an exponential separation between the sizes of LS-backdoors and strong backdoors. 
Their key insight was that if a CDCL solver without restarts queried
the $x_1,\ldots,x_n$ variables in the lexicographic ordering of
assignments, it would learn crucial conflict clauses that would enable
the solver to establish the unsatisfiability of the instance without
having to query any additional variables. (By the term ``querying a
variable'' we mean that the solver assigns value to it and then performs any unit propagations.)
Since strong backdoors cannot benefit from clause learning they will necessarily have to query additional variables to hit any conflict.

We show that the same family of formulas (but for a different ordering ${\cal O}$) can be used to separate LS-backdoor size from LSR-backdoor size.
Observe that for \emph{any} ordering ${\cal O}$ the variables $x_1, x_2, \ldots, x_n$ form an LSR-backdoor for ${\cal F}_{\cal O}$.
\begin{lemma}\label{lem:lsr-backdoor}
  Let ${\cal O}$ be any ordering of $\set{0,1}^n$.
  The $X$-variables form an LSR-backdoor for the formula ${\cal F}_{\cal O}$.
\end{lemma}
\begin{proof}
  For each assignment $\alpha \in \set{0,1}^n$ (ordered by ${\cal O}$), assign $\alpha$ to the $x$ variables by decision queries.
  By the structure of ${\cal F}_{\cal O}$, as soon as we have a complete assignment to the $x$ variables, we will unit-propagate to a conflict and learn a $q_\alpha$ variable as a conflict clause; after that we restart.
  Once all of these assignments are explored we will have learned the unit clause $q_\alpha$ for \emph{every} assignment $\alpha$, and so we can just query the $X$ variables in any order (without restarts) to yield a contradiction, since every assignment to the $X$ variables will falsify the formula.
\end{proof}

Note that the formula ${\cal F}_{\cal O}$ depends on $N = O(2^n)$ variables, and so the size of this LSR-backdoor is $O(\log N)$.
Furthermore, observe that the $X$-variables will also form an LS-backdoor if we can query the assignments $\alpha \in \set{0,1}^n$ according to ${\cal O}$ without needing to restart --- for example, if ${\cal O}$ is the lexicographic ordering.
This suggests the following definition, which captures the orderings ${\cal O}$ of $\set{0,1}^n$ that can be explored by a CDCL algorithm without restarts:
\begin{definition}
  Let ${\cal T}_X$ be the collection of all depth-$n$ decision trees on $X$-variables, where we label each leaf $\ell$ of a tree $T \in {\cal T}_X$ with the assignment $\alpha \in \set{0,1}^n$ obtained by taking the assignments to the $X$-variables on the path from the root of $T$ to $\ell$.
  For any $T \in {\cal T}_X$, let ${\cal O}(T)$ be the ordering of $\set{0,1}^n$ obtained by reading the assignments labelling the leaves of $T$ from left to right.
\end{definition}
To get some intuition for our lower-bound argument, consider an ordering ${\cal O}(T)$ for some decision tree $T \in {\cal T}_X$.
By using the argument in Lemma \ref{lem:lsr-backdoor} the formula ${\cal F}_{{\cal O}(T)}$ will have a small LS-backdoor, obtained by querying the $X$-variables according to the decision tree $T$.
Now, take any two assignments $\alpha_i, \alpha_j \in {\cal O}(T)$ and let ${\cal O}(T)'$ be the ordering obtained from ${\cal O}(T)$ by swapping the indices of $\alpha_i$ and $\alpha_j$.
If we try and execute the same CDCL algorithm without restarts (corresponding to the ordering ${\cal O}(T)$) on the new formula ${\cal F}_{{\cal O}(T)'}$, the algorithm will reach an inconclusive state once it reaches the clause corresponding to $\alpha_j$ in ${\cal O}(T)'$ since at that point the assignment to the $X$-variables will be $\alpha_i$.
Thus, it will have to query at least one more $q$ variable (for instance, $q_{\alpha_j}$), which increases the size of the backdoor by one.
We can generalize the above argument to multiple ``swaps'' --- the CDCL algorithm without restarts querying the variables according to ${\cal O}(T)$ would then have to query one extra variable for every $q_\alpha$ which is ``out-of-order'' with respect to ${\cal O}(T)$.

This discussion leads us to the following complexity measure: for any ordering ${\cal O} \in \set{0,1}^n$ (not necessarily obtained from a decision tree $T \in {\cal T}_X$) and any ordering of the form ${\cal O}(T)$, let \[ d({\cal O}, {\cal O}(T)) = |\set{\alpha' \in \set{0,1}^n \st \exists \alpha \in \set{0,1}^n : \alpha' <_{\cal O} \alpha, \alpha <_{{\cal O}(T)} \alpha'}|.\]
Informally, $d({\cal O}, {\cal O}(T))$ counts the number of elements of ${\cal O}$ which are ``out-of-order'' with respect to ${\cal O}(T)$ as we have discussed above.
We are able to show that the above argument is fully general:

\begin{lemma}\label{lem:main}
  Let ${\cal O}$ be any ordering of $\set{0,1}^n$, and let ${\cal T}_X$ denote the collection of all complete depth-$n$ decision trees on $X$ variables.
  Then any learning-sensitive backdoor of ${\cal F}_{\cal O}$ has size at least \[\min_{T \in {\cal T}_X} d({\cal O}, {\cal O}(T)).\]
\end{lemma}

This reduces our problem to finding an ordering ${\cal O}$ for which \emph{every} ordering of the form ${\cal O}(T)$ has many elements which are ``out-of-order'' with respect to ${\cal O}$ (again, intuitively for every mis-ordered element in the LS-backdoor we will have to query at least one more $q$-variable.)
\begin{lemma}
  For any $n > 4$ there exists an ordering ${\cal O}$ of $\set{0,1}^n$ such that for every decision tree $T \in {\cal T}_X$ we have \[ d({\cal O}, {\cal O}(T)) \geq 2^{n-2}.\]
\end{lemma}
\begin{sketch}
  We define the ordering, and leave the full proof of correctness to the companion technical report~\cite{extended_paper_link}.
  Let $\beta_1, \beta_2, \ldots, \beta_N$ be the lexicographic ordering of $\set{0,1}^n$, and for any string $\beta_i$ define $\overline{\beta_i}$ to be the string obtained by flipping each bit in $\beta_i$.
  Then let ${\cal O}$ be the ordering \[ \beta_1, \overline \beta_1, \beta_2, \overline \beta_2, \ldots, \beta_{N/2}, \overline \beta_{N/2}.\]
\end{sketch}

\begin{theorem}
For every $n > 3$, there is a formula $F_n$ on $N = O(2^n)$ variables such that the minimal LSR-backdoor has $O(\log N)$ variables, but every LS-backdoor has $\Omega(N)$ variables.
	\label{thm_ls_vs_lsr}
\end{theorem}



\section{Future Work and Conclusions}

We presented a large-scale study of several structural parameters of
SAT instances. We showed that combinations of these features can lead
to improved regression models, and that in general, industrial
instances tend to have much more favorable parameter values than
crafted or random. Further, we gave examples of how such parameters
may be used as a ``lens'' to evaluate different solver heuristics. We
hope these studies can be used by complexity theorists as a guide for
which parameters to focus on in future analyses of CDCL. Finally, we
introduced LSR-backdoors, which characterize a sufficient subset of
the variables that CDCL can branch on to solve the formula.  In doing
so, we presented an algorithm to compute LSR-backdoors that exploits
the notion of clause absorption, and further showed that certain
formulas have exponentially smaller minimal LSR-backdoors than
LS-backdoors under the assumption that the CDCL solver only backtracks
(not backjumps) and uses the 1st-UIP learning scheme. From a theoretical point
of view, it remains open whether there is a separation between LS and
LSR-backdoors when backjumping is allowed.  On the empirical side, we
plan to investigate approaches to computing many small LSR-backdoors
to formulas. The intuition is that CDCL solvers may be more efficient
on Application instances because they have lots of LSR-backdoors and
the solver may ``latch'' on to a backdoor relatively easily, while
these solvers are less efficient for crafted instances because they
have too few backdoors.  Further, we plan to refine our results by
analyzing individual sub-categories of the benchmarks studied,
with a particular focus on scaling to crafted instances.

\bibliographystyle{splncs03}
\bibliography{refs}

\appendix
\newpage
\section{LSR Notation Example}
\label{app_lsr_notation}
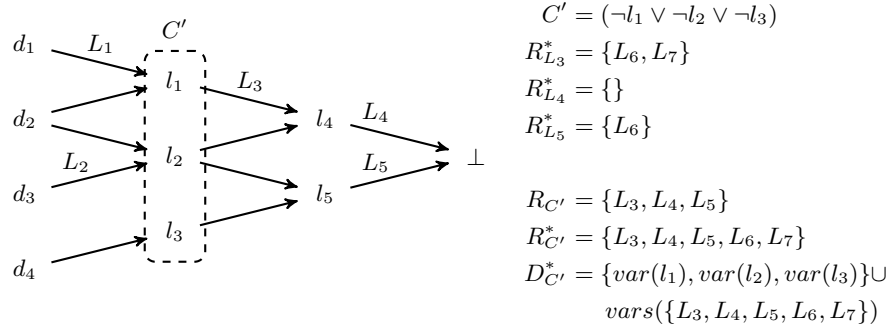
\begin{figure}
	\begin{minipage}{0.6\linewidth}
	\begin{center}
	\begin{tikzpicture}[->,>=stealth',shorten >=1pt,auto,node distance=3cm,
	thick,main node/.style={circle,draw=none},decision node/.style={circle,draw}] 

	\node[main node] at (0,3) (d1) {$d_1$};
	\node[main node] at (0,2) (d2) {$d_2$};
	\node[main node] at (0,1) (d3) {$d_3$};
	\node[main node] at (0,0) (d4) {$d_4$};
	
	\node[main node] at (2, 2.5) (l1) {$l_1$};
	\node[main node] at (2, 1.5) (l2) {$l_2$};
	\node[main node] at (2, 0.5) (l3) {$l_3$};
	\node[main node] at (2, 3.2) (cprime) {$C^\prime$};
	
	\node[main node] at (4, 2) (l4) {$l_4$};
	\node[main node] at (4, 1) (l5) {$l_5$};
	
	\node[main node] at (6, 1.5) (bot) {$\bot$};

	\draw[] (d1) -- (l1) node [midway, above] {$L_1$};
	\draw[] (d2) -- (l1);
	\draw[] (d2) -- (l2);
	\draw[] (d3) -- (l2) node [near start, above] {$L_2$};
	\draw[] (d4) -- (l3);

	\draw[] (l1) -- (l4) node [midway, above] {$L_3$};
	\draw[] (l2) -- (l4);
	\draw[] (l2) -- (l5);
	\draw[] (l3) -- (l5);
	
	\draw[] (l4) -- (bot) node [near start, above] {$L_4$};
	\draw[] (l5) -- (bot) node [near start, above] {$L_5$};

	\node[rectangle,inner sep=0.6mm,rounded corners, dashed,draw, fit=(l1) (l2) (l3)](lc) {};
	\end{tikzpicture}
	\end{center}
	\end{minipage}%
	\begin{minipage}{0.34\linewidth}
	\begin{align*}
	C^\prime &= (\neg l_1 \vee \neg l_2 \vee \neg l_3) \\
	R_{L_3}^* &= \{L_6, L_7\}\\
	R_{L_4}^* &= \{\}\\
	R_{L_5}^* &= \{L_6\}\\
	\\
	R_{C^\prime} &= \{L_3, L_4, L_5\}\\
	R_{C^\prime}^* &= \{L_3, L_4, L_5, L_6, L_7\}\\
	D_{C^\prime}^* &= \{var(l_1), var(l_2), var(l_3)\} \cup \\
	& \hphantom{asd} vars(\{L_3, L_4, L_5, L_6, L_7\})
	\end{align*}
	\end{minipage}
	\caption{Example conflict analysis graph depicting the set of relevant clauses and variables to some learnt clause $C^\prime$. Nodes are literals. Edges labeled with some $L_i$ are previously learnt clauses; all other edges depicting propagations are from the original formula $F$. The clauses $L_6, L_7$ used to derive $L_3$ and $L_5$ are not shown, but would be in the respective conflict analysis graphs of $L_3$ and $L_5$. The clauses $L_1$ and $L_2$ are not included in $R_{C^\prime}$ since they occur on the reason side of the graph. }
	\label{notation-figure}
	\end{figure}

\section{Full Proofs on Computing LSR-Backdoors}
\label{sec_appendix_proofs}

In this appendix, we present the proofs for Lemma \ref{lem_absorb_clauses} and Theorem \ref{thm_lsr_sat}. We define the following notation. We let $F$ be a formula, overloaded to also denote its set of clauses. Let $\Delta$ be a set of clauses and $C$ be some clause we would like to absorb. We define the function $absorb(\Delta, C) \rightarrow \Delta_2$, which produces a new clause set $\Delta_2$, such that $\Delta \subseteq \Delta_2$ and $\Delta_2$ absorbs $C$ by applying Lemma \ref{absorb-lemma}. If $C$ is already absorbed by $\Delta$, then $absorb(\Delta, C) = \Delta$, and if $C$ is not 1-provable with respect to $\Delta$, then $absorb(\Delta, C) =$ \textit{fail}. We overload $absorb$ to take a sequence of clauses, such that $absorb(\Delta, \langle C_1, C_2, \ldots, C_n \rangle) = \Delta_n$, which applies $absorb$ to the clauses in order, and for each intermediate $\Delta_i$ produced, it absorbs every clause $C_j, j \le i$. Again, if any clause $C_i$ is not 1-provable with respect to $\Delta_{i-1}$, then $absorb$ returns \textit{fail}.

	\begin{customlemma}{2}	
		Let $S$ be a CDCL solver used to determine the satisfiability of some formula $F$. Let $\Delta \subseteq S_L$ be a set of clauses learned during solving. Then a fresh solver $S^\prime$ can absorb all clauses in $\Delta$ by only branching on the variables in $D_{\Delta}^*$.
	\end{customlemma}
	\begin{proof}
		We show that $S^\prime$ can absorb all clauses in $R_\Delta^*$, which includes $\Delta$. Let $seq(R_\Delta^*)=\langle C_1, C_2, \ldots C_n \rangle$ be the sequence over $R_\Delta^*$ in the order that the original solver $S$ derived the clauses.
		Consider the first clause $C_1$. By construction, it does not depend upon any learnt clauses (i.e. it was derived from only original clauses), and since $S$ learned $C_1$, it must be 1-empowering and 1-provable with respect to the initial clause set. By Lemma \ref{absorb-lemma}, we can absorb $C_1$ by only branching on variables in $C_1$, which again by construction are in $D_\Delta^*$. We therefore have that $absorb(F, C_1) = \Theta_1$ absorbs $C_1$. 
		
		Suppose $absorb(F, \langle C_1, C_2, \ldots C_{k-1} \rangle) = \Theta_{k-1}$ absorbs the first $k-1$ clauses by only branching on the variables in $C_1, C_2, \ldots C_{k-1}$, and we wish to absorb $C_k$. There are two cases to consider. First, $C_k$ may already be absorbed by $\Theta_{k-1}$, since the clauses learnt by $absorb(\ldots)$ may absorb clauses in addition to $C_1, \ldots, C_{k-1}$, in which case we are done. So suppose $C_k$ is not absorbed by $\Theta_{k-1}$. Since every previous clause in $seq(R_\Delta^*)$ has been absorbed, we in particular have that the clauses in $R_{C_k}$ have been absorbed, so $C_k$ must be 1-provable. To see this, suppose instead of absorbing $R_{C_k}$ we learned the exact set of clauses in $R_{C_k}$. Then by construction, negating all literals in $C_k$ must lead to a conflict through unit propagation. Since we have instead absorbed $R_{C_k}$, any propagation that was used to derive the conflict must also be possible using the clauses that absorb $R_{C_k}$ (by definition of absorption).
		
		We also know that $C_k$ is 1-empowering with respect ot $\Theta_{k-1}$, since otherwise it is absorbed by definition, and we assumed this is not true.
		Therefore, we can invoke Lemma \ref{absorb-lemma}, such that $absorb(\Theta_{k-1}, C_k)=\Theta_k$,  which is derived by only branching on the variables in $C_k$. Again, $vars(C_k) \subseteq D_\Delta^*$ by construction.
	\end{proof}
	
	\begin{customthm}{1}[LSR Computation, SAT case]
		Let $S$ be a CDCL solver, $F$ be a satisfiable formula, and $T$ be the final trail of the solver immediately before returning SAT, which is composed of a sequence of decision variables $T_D$ and propagated variables $T_P$. For each $p\in T_P$, let the clause used to unit propagate $p$ be $l_p$ and the full set of such clauses be $L_P$.  Then $B = T_D \cup D_{L_p}^*$ constitutes an LSR-backdoor for $F$.
	\end{customthm}
	\begin{proof}
		Using Lemma \ref{lem_absorb_clauses}, we first absorb all clauses in $L_P$ by branching on $D_{L_P}^*$. 
		We can then restart the solver to clear the trail, and branch on the variables in $T_D$, using the same order and polarity as the final trail of $S$. If any $d \in T_D$ is already assigned due to learnt clauses used to absorb $L_{P}$, unit propagation will be able derive the literals propagated by $d$, since we have absorbed $l_d$. Note that with this final branching scheme, we can not reach a state where the wrong polarity of a variable in $T_D$ becomes implied through propagation (i.e. with respect to the final trail polarities), since the solver is sound and this would block the model found by the original solver $S$. 
	\end{proof}

\section{Proof Separating LS and LSR-Backdoors}
Here we present the full proofs for Section \ref{sec_lsr_theory}.

Let $n$ be a positive integer and let $X = \set{x_1, x_2, \ldots, x_n}$ be a set of boolean variables.
For any boolean variable $x$, let $x^1$ denote the positive literal $x$ and $x^{0}$ denote the negative literal $\neg x$.
For any assignment $\alpha \in \set{0,1}^n$ let $C_\alpha = x_1^{1-\alpha_1} \vee x_2^{1-\alpha_2} \vee \ldots \vee x_n^{1-\alpha_n}$ denote the clause on $x$ variables which is uniquely falsified by the assignment $\alpha$.
\suppress{
The following CNF formula ${\cal F}$ was introduced for the purpose of separating \emph{strong backdoors} from \emph{learning sensitive backdoors}.

The formula is defined on the variables $x_1, x_2, \ldots, x_n$, and also three auxiliary sets of variables $\set{q_\alpha}_{\alpha \in \set{0,1}^n}, \set{a_\alpha}_{\alpha \in \set{0,1}^n}, \set{b_\alpha}_{\alpha \in \set{0,1}^n}$:
\begin{align*}
  {\cal F} = & \bigwedge_{\alpha \neq \vec 1} (C_\alpha \vee \neg q_\alpha) \wedge (q_\alpha \vee a_\alpha) \wedge (q_\alpha \vee b_\alpha) \wedge (q_\alpha \vee \neg a_\alpha \vee \neg b_\alpha) \\
  & \wedge \left(C_{\vec 1} \vee \bigvee_{\alpha \in \set{0,1}^n} \neg q_\alpha\right) \wedge (q_{\vec 1} \vee a_{\vec 1}) \wedge (q_{\vec 1} \vee b_{\vec 1}) \wedge (q_{\vec 1} \vee \neg a_{\vec 1} \vee \neg b_{\vec 1}).
\end{align*}
In this note we modify this formula to obtain a CNF formula providing an exponential separation between \emph{learning sensitive backdoors} and \emph{learning-sensitive backdoors with restarts}.
}
Throughout we implicitly assume that the clause learning algorithm is \emph{1st-UIP}.

Our formula will be defined by the following template.
Let ${\cal O}$ be an ordering of $\set{0,1}^n$; we write $<_{\cal O}$ to denote the relation induced by this ordering.
Given an ordering ${\cal O}$ we can define the formula \[{\cal F}_{\cal O} = \bigwedge_{\alpha \in \set{0,1}^n} (C_\alpha \vee \bigvee_{\alpha' \leq_{\cal O} \alpha} \neg q_{\alpha'}) \wedge (q_\alpha \vee a_\alpha) \wedge (q_\alpha \vee b_\alpha) \wedge (q_\alpha \vee \neg a_\alpha \vee \neg b_\alpha).\]

\begin{customlemma}{3}
  Let ${\cal O}$ be any ordering of $\set{0,1}^n$.
  The $X$-variables form a learning-sensitive backdoor with restarts for the formula ${\cal F}_{\cal O}$.
\end{customlemma}
\begin{proof}
  Query the $x$ variables according to the ordering given by ${\cal O}$.
  As soon as we have a complete assignment to the $x$ variables, we will unit-propagate to a conflict and learn a $q_\alpha$ variable as a conflict clause; after that we restart.
  Once all such assignments are explored we can simply query the $X$ variables in any order (without restarts) to yield a contradiction, since every assignment to the $X$ variables will falsify the formula.
\end{proof}

Consider any decision tree $T$ of depth $n$ where we have queried the $X$ variables in \emph{any} order in the tree $T$.
From this tree we obtain a natural ordering ${\cal O}(T)$ of $\set{0,1}^n$ by reading off the strings labelling the leaves in left-to-right order.
Note that the orderings of the form ${\cal O}(T)$ are \emph{exactly} the orderings which can be generated by DPLL algorithms (or, more generally, a CDCL algorithm without restarts).
For any such complete decision tree $T$ and any ordering ${\cal O}$ of $\set{0,1}^n$, define \[ d({\cal O}, {\cal O}(T)) = |\set{\alpha' \in \set{0,1}^n \st \exists \alpha \in \set{0,1}^n : \alpha' <_{\cal O} \alpha, \alpha <_{{\cal O}(T)} \alpha'}|.\]
That is, $d({\cal O}, {\cal O}(T))$ is the number of strings $\alpha'$ which are ``out of order'' in ${\cal O}(T)$ with respect to ${\cal O}$.

\begin{customlemma}{4}\label{lem:main}
  Let ${\cal O}$ be any ordering of $\set{0,1}^n$, and let ${\cal T}_X$ denote the collection of all complete depth-$n$ decision trees on $X$ variables.
  Then any learning-sensitive backdoor of ${\cal F}_{\cal O}$ has size at least \[\min_{T \in {\cal T}_X} d({\cal O}, {\cal O}(T)).\]
\end{customlemma}
\begin{proof}
  Let $B$ be a minimal learning-sensitive backdoor of ${\cal F}_{\cal O}$.
  Let $Q = \set{q_\alpha}_{\alpha \in \set{0,1}^n}$.
  First we show that without loss of generality the following holds:
  \begin{enumerate}
  \item $B \subseteq X \cup Q$.
  \item All queries to $Q$-variables occur before any query to an $X$-variable.
  \end{enumerate}
  We begin with an observation: any query to a variable $q_\alpha$ can be assumed to be of the form $q_\alpha = F$.
  To see this, notice that querying $q_\alpha = F$ will immediately unit propagate to a conflict, and 1UIP-learning will immediately yield the unit clause $q_\alpha$.
  Thus we can always replace queries of the form $q_\alpha = T$ \emph{in-situ} with queries $q_\alpha = F$ without affecting the rest of the algorithm's execution.

  Let us first show that $B \subseteq X \cup Q$.
  Suppose that there is a variable $a_\alpha \in B$ (the case where $b_\alpha \in B$ is symmetric).
  We show that $a_\alpha$ can be replaced with $q_\alpha$.
  Consider the first time that $a_\alpha$ is queried as a decision variable.
  By the structure of ${\cal F}_{\cal O}$, we can assume that $q_\alpha$ has not been assigned before querying $a_\alpha$ as a decision variable (as we have shown above, if this is the case then $q_\alpha = T$, which eliminates the two clauses containing $a_\alpha$).
  If $a_\alpha = F$, then $q_\alpha = T$ is unit-propagated, and all clauses containing $a_\alpha$ or $b_\alpha$ will be removed.
  It follows that any conflict following this assignment must occur in a clause containing $C_{\alpha'}$ for some $\alpha'$, and thus replacing $a_\alpha$ with $q_\alpha$ will not affect the queries to the backdoor in this case.
  So, instead, suppose that $a_\alpha = T$.
  Clearly, in order for this assignment to have been necessary, the variable $q_\alpha$ must be assigned after assigning $a_\alpha$ (either by unit propagation or a decision).
  Observe that if $q_\alpha$ is set to true after assigning $a_\alpha = T$, then the two clauses containing $a_\alpha$ are removed and so assigning $q_\alpha = T$ to begin with would not have affected the execution of the algorithm on the backdoor.
  Similarly, setting $q_\alpha = F$ will unit propagate to a conflict in this case and we would learn the unit clause $q_\alpha = T$ as above, again showing that we could have replaced $a_\alpha$ with $q_\alpha$ without affecting the backdoor.
  Thus $a_\alpha$ can be replaced with $q_\alpha$ without loss of generality.

  Next we show that all decision queries to $Q$-variables can be made before queries to $X$-variables without loss of generality.
  As argued above, we can assume that when $q_\alpha$ is queried as a decision variable it is queried negatively (as this will lead immediately to learning the unit clause $q_\alpha$).
  Consider any conflict which occurs in the process of querying the variables of the backdoor that does not happen because of querying $\neg q_\alpha$.
  Since $B \subseteq X \cup Q$, by the structure of ${\cal F}_{\cal O}$ any such conflict must occur after assigning all $X$-variables to some string $\alpha \in \set{0,1}^n$.
  Corresponding to this total assignment $\alpha$ to the $X$-variables is the unique clause $C_\alpha$ which is falsified by this assignment, so consider the clause in ${\cal F}_{\cal O}$ of the form \[ C_\alpha \vee \bigvee_{\alpha' \leq_{\cal O} \alpha} \neg q_{\alpha'}.\]
  After assigning a subset of the $Q$-variables of $B$ and assigning all $X$-variables to $\alpha$ a conflict can occur in two ways:
  \begin{enumerate}
  \item $q_{\alpha'} = 1$ for all $\alpha' \leq_{\cal O} \alpha$ and the clause above is the conflict clause, or
  \item there exists a single $q_{\alpha^*}$ which is unassigned after assigning all of these variables; the above clause leads to assigning $q_{\alpha^*} = F$ and thus learning the unit clause $q_{\alpha^*}$, as argued above.
  \end{enumerate}
  In either case, moving all $q_\alpha$ queries to the beginning of the algorithm will not affect these conflicts: this is clear in case (1); in case (2) the only possibility is that $q_{\alpha^*}$ was queried at the beginning and so the conflict will change to a case (1) conflict.

  We are now ready to finish the proof of the lemma.
  We encode the execution of the CDCL algorithm as a decision tree, wherein we first query all $Q$-variables and then query all $X$-variables.
  Note that we must assign all $X$-variables to hit a conflict by the structure of ${\cal F}_{\cal O}$, and so we let $T$ denote the complete depth-$n$ decision tree querying the $X$-variables obtained from the CDCL execution tree.
  Recalling that $B \subseteq X \cup Q$, to prove the lemma it suffices to prove that \[|B \cap Q| \geq d({\cal O}, {\cal O}(T)).\]
  
  Since $B$ is a backdoor, it must be that at every leaf of the tree $T$ the unit propagator will propagate the input to a conflict.
  Consider the ordering of the assignments ${\cal O}(T)$ induced by this tree $T$.
  At each leaf of this tree, which corresponds to an assignment $\alpha$ to the $X$ variables, consider the clause \[ C_\alpha \vee \bigvee_{\alpha' \leq_{\cal O} \alpha} \neg q_{\alpha'}.\]
  There are two possibilities: either we learn a unit clause $q_{\alpha^*}$ for some $\alpha^* \leq_{\cal O} \alpha$ or we do not.
  If we have not then we must have already learned the unit clauses $q_{\alpha'}$ for all $\alpha' \leq_{\cal O} \alpha$.
  Otherwise, we must have learned the unit clauses $q_{\alpha'}$ for each $\alpha' \leq_{{\cal O}} \alpha$ satisfying $\alpha' \neq \alpha^*$.
  This implies that \[ B \cap Q \supseteq \bigcup_{\alpha \in \set{0,1}^n} \set{\alpha' \st \alpha' >_{{\cal O}(T)} \alpha, \alpha' \leq_{\cal O} \alpha},\] as for each $\alpha \in \set{0,1}^n$ we must query, at least, all variables $\alpha'$ which occur after $\alpha$ in ${\cal O}(T)$ but before $\alpha$ in ${\cal O}$.
  This proves the lemma.
\end{proof}

By the previous lemma, to lower bound the size of the learning sensitive backdoor it suffices to find an ordering ${\cal O}$ of $\set{0,1}^n$ for which any ordering ${\cal O}(T)$ produced by a decision tree has many ``inversions'' with respect to ${\cal O}$.

\begin{customlemma}{5}
  There exists an ordering ${\cal O}$ of $\set{0,1}^n$ such that for every decision tree $T \in {\cal T}_X$ we have \[ d({\cal O}, {\cal O}(T)) \geq 2^{n-2}.\]
\end{customlemma}
\begin{proof}
  Let $T \in {\cal T}_X$ be any decision tree.
  Let $N = 2^n$ and write ${\cal O}(T) = \alpha_1, \alpha_1, \ldots, \alpha_N$.
  We use the following key property of orderings generated from decision trees:

  \spcnoindent
  {\bf Key Property.} For any decision tree $T$ there is a coordinate $i \in [n]$ and $b \in \set{0,1}$ such that \[ \alpha_j[i] = b, \alpha_k[i] = 1-b\] for all $j = 1, 2, \ldots, N/2$ and $k = N/2 + 1, N/2 + 2, \ldots, N$.

  \spcnoindent
  This property is easy to prove: simply let $i$ be the index of the variable labelled on the root of the decision tree $T$.
  If the decision tree queries the bit $b$ in the left subtree then the first half of the strings in the ordering will have the $i$th bit of the string set to $b$, and the second half of strings set to $1-b$.

  To use this property, let $\beta_1, \beta_2, \ldots, \beta_N$ be the lexicographic ordering of $\set{0,1}^n$, and for any string $\beta_i$ define $\overline{\beta_i}$ to be the string obtained by flipping each bit in $\beta_i$.
  Then let ${\cal O}$ be the ordering \[ \beta_1, \overline \beta_1, \beta_2, \overline \beta_2, \ldots, \beta_{N/2}, \overline \beta_{N/2}.\]
  It follows that for each $i \in [n]$ and each $b \in \set{0,1}$, half of the strings $\alpha \in \set{0,1}^n$ with $\alpha[i] = b$ will be in the first half of the ordering ${\cal O}$ and half of the strings with $\alpha[i] = b$ are in the second half of the ordering ${\cal O}$.
  The lemma follows by applying the Key Property.
\end{proof}

\begin{corollary}
  There exists an ordering ${\cal O}$ such that any learning-sensitive backdoor for the formula ${\cal F}_{\cal O}$ has size at least $2^{n-2}$.
\end{corollary}

\begin{customthm}{3}
	For every $n > 3$, there is a formula $F_n$ on $N = O(2^n)$ variables such that the minimal LSR-backdoor has $O(\log N)$ variables, but every LS-backdoor has $\Omega(N)$ variables.
\end{customthm}

\end{document}